\documentclass[hyperref]{article}

\usepackage{microtype}
\usepackage{graphicx}
\usepackage{subfigure}
\usepackage{booktabs} %

\usepackage{hyperref}

\usepackage[accepted]{icml2023}

\usepackage[utf8]{inputenc} %
\usepackage[T1]{fontenc}    %
\usepackage{hyperref}       %
\usepackage{url}            %
\usepackage{booktabs}       %
\usepackage{amsfonts}       %
\usepackage{nicefrac}       %
\usepackage{microtype}      %
\usepackage{xcolor}         %
\usepackage[normalem]{ulem}
\usepackage{comment}
\usepackage{yfonts}
\usepackage{amssymb}
\usepackage{bm}
\usepackage{bbm}
\usepackage{dsfont}
\usepackage{natbib}
\usepackage{graphicx}
\usepackage{amsthm}
\usepackage{amssymb}
\usepackage{amsmath}
\usepackage{mathtools}
\usepackage{mathrsfs}

\DeclareMathOperator*{\argmax}{arg\,max}
\DeclareMathOperator*{\argmin}{arg\,min}
\DeclareMathOperator{\GC}{GC}

\DeclareMathOperator{\Var}{Var}

\DeclareMathOperator{\erf}{erf}
\DeclareMathOperator{\PI}{PI}

\DeclareMathOperator{\dist}{dist}

\usepackage[capitalize,noabbrev]{cleveref}

\theoremstyle{plain}
\newtheorem{theorem}{Theorem}[section]
\newtheorem{proposition}[theorem]{Proposition}
\newtheorem{lemma}[theorem]{Lemma}

\theoremstyle{definition}
\newtheorem{definition}[theorem]{Definition}

\theoremstyle{remark}

\icmltitlerunning{A Margin-based Multiclass Generalization Bound via Geometric Complexity}

\begin{document}

\twocolumn[
\icmltitle{A Margin-based Multiclass Generalization
Bound via Geometric Complexity}

\icmlsetsymbol{equal}{*}

\begin{icmlauthorlist}
\icmlauthor{Michael Munn}{equal,goog}
\icmlauthor{Benoit Dherin}{equal,goog}
\icmlauthor{Javier Gonzalvo}{goog}
\end{icmlauthorlist}

\icmlaffiliation{goog}{Google Researh, USA}

\icmlcorrespondingauthor{Michael Munn}{munn@google.com}

\icmlkeywords{learning theory, generalization bound, model complexity, deep learning theory, geometric complexity}

\vskip 0.3in
]

\printAffiliationsAndNotice{\icmlEqualContribution} %

\begin{abstract}
There has been considerable effort to better understand the generalization capabilities of deep neural networks both as a means to unlock a theoretical understanding of their success as well as providing directions for further improvements. In this paper we investigate margin-based multiclass generalization bounds for neural networks which rely on a recent complexity measure, the geometric complexity,  developed for neural networks and which measures the variability of the model function \cite{dherin2022neural}.

We derive a new upper bound on the generalization error which scales with the margin-normalized geometric complexity of the network and which holds for a broad family of data distributions and model classes. Our generalization bound is empirically investigated for a ResNet-18 model trained with SGD on the CIFAR-10 and CIFAR-100 datasets with both original and random labels. 
\end{abstract}

\section{Introduction}
Within the field of machine learning, a model's ability to generalize well to unseen data is one of the key metrics of performance. Modern deep learning techniques have proven successful at this goal across multiple domains and applications. However, it is still not well understood how or why these neural network models exhibit such good generalization capabilities. 

Classical statistical learning theory provides a theoretical framework to understand generalization and over the years various complexity measures have been proposed which aim to capture the relationship between generalization and complexity. In this context, the expectation is that lower model complexity should imply tighter generalization gaps. However, many of these well-known complexity measures, such as parameter count or parameter norms are not well-suited to the experimental results that are often observed when training neural networks \cite{jiang2019fantastic, zhang2017understanding}. In particular, the phenomena of double descent \cite{nakkiran2021deep, Belkin15849, belkin2021fear} clearly illustrates the drawbacks of relying on simple parameter count or model depth to measure model complexity and shows that as model size increases the model is able to perfectly fit the training data while also obtaining low test error. Other more theoretically motivated complexity measures may provide generalization error bounds which are vacuous or computationally intractable to compute. We refer the reader to \cite{jiang2019fantastic} which provides a well-written discussion of the landscape of current and commonly used complexity measures as well as an extensive empirical analysis of over 40 different generalization measures. 

In this paper we focus on a complexity measure recently proposed in \cite{dherin2021geometric, dherin2022neural}. In those papers, the authors show through theoretical and empirical techniques that the geometric complexity is well-suited for the analysis of deep neural networks. Namely, they show that a large number of training and tuning heuristics in deep learning have the advantageous side-effect of decreasing the geometric complexity of the learned solution. These results hint that the geometric complexity can serve as a useful proxy for measuring network performance. 
Supporting this idea, thorough experimentation in \cite{novak2018sensitivity} has shown empirically that a complexity measure similar to the geometric complexity correlates strongly with model generalization.

\subsection{Contributions}
Given these positive indicators, one natural question to explore is whether the geometric complexity can serve as an effective means for bounding the generalization error in neural networks. In this paper we show this to be true, and derive new upper bounds on the margin-based multi-class generalization error which scale with the margin-normalized
geometric complexity of the network (Theorem \ref{theorem:binary_generalization_bound}). 

This relationship can best be illustrated in Figure \ref{figure:cifar10} which plots the excess risk (test accuracy minus the training accuracy) across multiple epochs of training a ResNet-18 model \cite{he2016deep} on the CIFAR-10 dataset \cite{krizhevsky2009learning}; c.f. \cite{bartlett2017spectrally} which examines a similar behavior for the Lipschitz constant of the network (i.e., the product of the spectral norms of the weight matrices) of AlexNet also trained on CIFAR-10. Similar to the behavior of the Lipschitz constant in that setting, note that the geometric complexity (see Definition \ref{definition:geometric_complexity} for the precise formulation) is correlated with the excess risk, both when training on original labels and random labels; see also Figure \ref{figure:cifar10_appendix} in Appendix  \ref{section:appendix_cifar_experiments} which includes plots demonstrating similar behavior when training ResNet-18 on the CIFAR-100 dataset.

Our contributions can be summarized as follows:
\begin{itemize}
    \item Theorem \ref{theorem:multiclass_generalization_bound} below states the main generalization bound that is the basis of this work. It is important to note that this bound has no dependence on artifacts of the network architecture such as number or depth of layers and is multi-class. The full details of the proof can be found in the Appendix section \ref{section:appendix_multiclass_proof}.
    \item Our proof relies on a novel covering number argument which ultimately follows from a consequence of an assumption on the underlying data distribution on which the model is trained. Namely, we require that the probability distribution from which data is sampled satisfies a Poincar\'e inequality (see Definition \ref{definition:prob_measure_satisfies_PI}). This framework further highlights the data dependence of the complexity measure as well as geometric properties of the model function. 
    \item The theorems we present here provide bounds on the generalization error which depend on the theoretical geometric complexity of the model function (see Definition \ref{definition:theoretical_geometric_complexity}). In Section \ref{section:theoretical_vs_empirical} we show how the empirical geometric complexity (which is computed over a sample) compares to the theoretical geometric complexity (which is measured over the entire data distribution). Namely we show that for Lipschitz functions these two quantities are comparable; see Proposition \ref{proposition:relate_theoretical_empirical_gc}.
\end{itemize}

The theorem below states that for a large class of data distributions (i.e., those that satisfy the Poincar\'e inequality which include the uniform distribution, Gaussian and mixtures of Gaussian distributions) the generalization error is bounded by the geometric complexity of the network:
\newpage
\begin{theorem}\label{theorem:multiclass_generalization_bound}
Given $a_1, a_2$ be positive reals. Let $S = \{(x_1,  y_1), \dots, (x_m, y_m)\}$ be i.i.d.~input-output pairs in $\mathbb{R}^d \times \{1,\cdots, k\}$ and suppose the distribution $\mu$ of the $x_i$ satisfies the Poincar\'e inequality with constant $\rho >0$.  Then, for any $\delta > 0$, with probability at least $1 - \delta$, every margin $\gamma > 0$ and network $f: \mathbb{R}^d \to \mathbb{R}^k$ which satisfies $\GC(f, \mu) \leq a_1$ and $\|\mathbb{E}_{\mu}(f)\| \leq a_2$ satisfy
$$
\mathbb{P}\left[\argmin_jf(x)_j \neq y \right] \leq \widehat{\mathcal{R}}_{S, \gamma}(f) + 
\dfrac{36\tilde{C}\sqrt{k\pi}}{\gamma m} + 
 3\sqrt{\dfrac{\log \frac{2}{\delta}}{2m}}
$$
where $\widehat{\mathcal{R}}_{S, \gamma}(f) =m^{-1} \sum_i \mathbbm{1}_{y_if(x_i) \leq \gamma}$ and $\tilde{C} = a_2 + \sqrt{a_1\rho/\delta}$.
\end{theorem}

\begin{figure}[h]
\includegraphics[width=8.4cm]{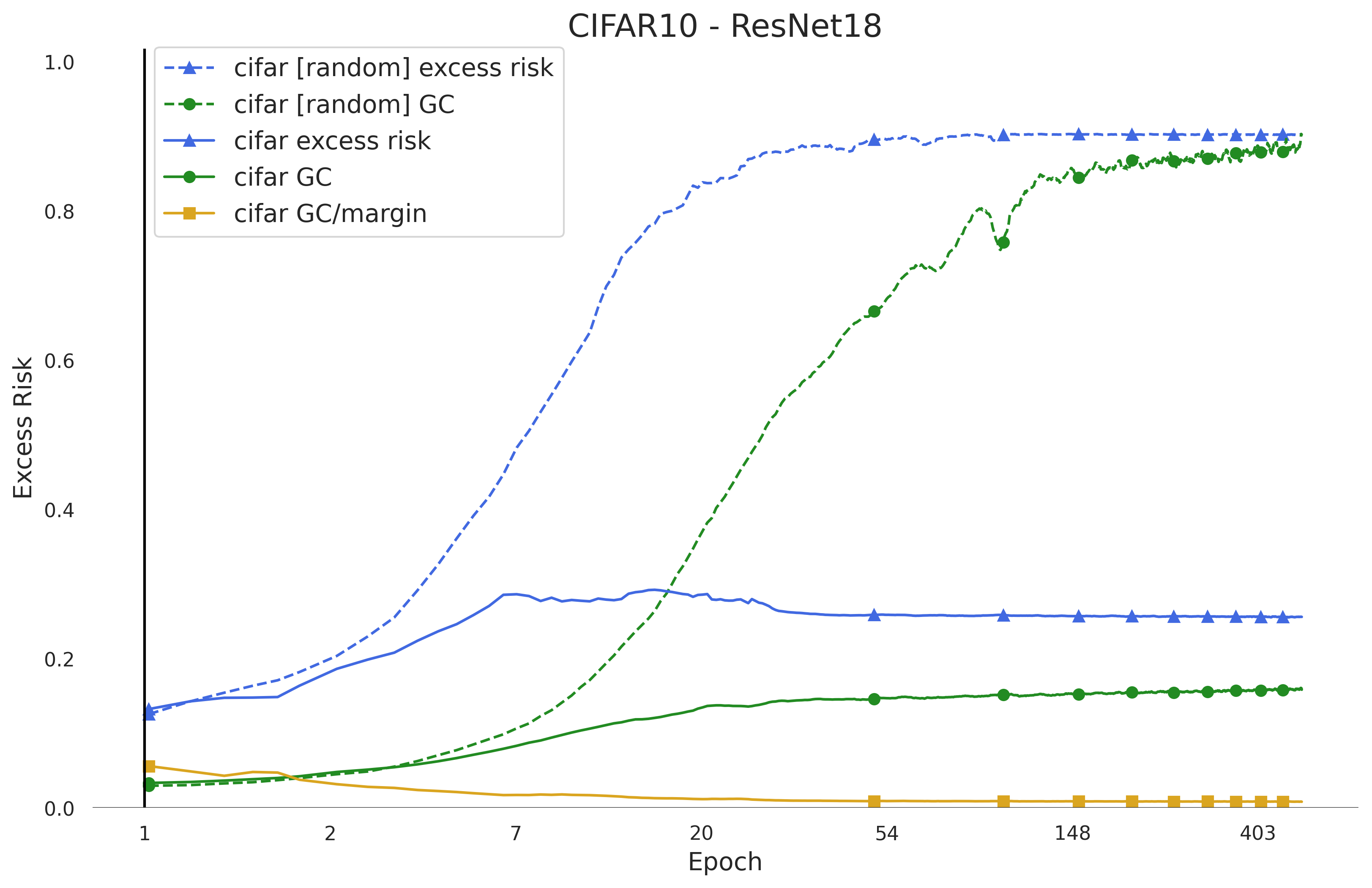}
\caption{Analysis of ResNet-18 \cite{he2016deep} trained with SGD on CIFAR-10 with both original and with random labels. The blue triangle-marked curves plot the excess risk across training epochs (on
a log scale) while the green circle-marked curves track the geometric complexity ($\GC$), normalized so that the two curves for random labels meet. Note that in both settings the $\GC$ is closely correlated with the excess risk. Furthermore, normalizing the $\GC$ by the margin (i.e., the
square-marked curve) neutralizes growth across epochs. Similar plots for CIFAR-100 can be found in Section \ref{section:appendix_cifar_experiments} of the Appendix.}
\label{figure:cifar10}
\end{figure}

\section{Related Work}
There has been a considerable amount of work in this direction over the years. Here we survey just a few of the most relevant results as they relate to the current work.

\vspace{-.15in}
\paragraph{Complexity measures for neural networks and double descent.} An interesting aspect of the geometric complexity, as discussed in detail in \cite{dherin2022neural}, is that it captures the double-descent phenomena \cite{Belkin15849,  belkin2021fear, nakkiran2021deep}. Namely, when training multiple ResNet-18 models on CIFAR-10 with increasing layer width, both the test loss and the $\GC$ follow a double descent curve as the model width increases; see Figure 5 in \cite{dherin2022neural}. Recent alternative measures of complexity for neural networks, which also appear to correlate with generalization, appear to capture this phenomenon as well \cite{grant2022predicting, gamba2022deep, achille2018emergence, novak2018sensitivity}. This suggests interesting connections between the double-descent phenomena and generalization bounds for neural networks.

\vspace{-.15in}
\paragraph{Implicit regularization and sharpness aware techniques.} The generalization power of neural network in spite of their high expressivity has suggested that a form of implicit regularization keeps the model from overfitting \cite{neyshabur2017implicit}. Recent work \cite{barrett2021implicit, smith2021on, ghosh2023implicit, chao2021sobolev} have shown that the optimization scheme implicitly regularizes the loss gradients favoring flat regions of the loss landscape. In turn these flatter regions with smaller gradients have been connected to smaller geometric complexity \cite{dherin2021geometric, dherin2022neural}. In \cite{zhang2023gradient} the increased generalization power of the sharpness-aware optimizer constructed in \cite{foret2021sharpnessaware} has been shown to rely of a similar mechanism of loss gradient regularization allowing the author to derive a generalization bound involving the loss gradients.

\vspace{-.15in}
\paragraph{Generalization bounds for neural networks.}
At a high level this work can be seen as a natural companion of the spectrally-normalized margin bounds obtained for neural networks in \cite{bartlett2017spectrally}. 
However, as their bound has clear dependence on the network architecture and norms of layer weights, ours is not directly comparable. While this lack of dependence on architecture can be viewed as an advantage of the $\GC$ as a complexity measure (primarily in its computation), our work should also be viewed in context of \cite{nagarajan2019uniform} which argues that such bounds alone cannot fully explain the generalization abilities of overparameterized neural networks. In particular, our bound does not reflect a width/depth dependence and, further, it remains to explore how our results perform in relation to the vacuity demonstrated by Nagarajan et al. We leave this detailed analysis for future work.

Finally, let us note that a number of other generalization bounds using different notions of complexity have been derived for neural networks; see for instance
\cite{zhang2023gradient,ghosh2023implicit,foret2021sharpnessaware,Sokolic2017RobustLM,Chatterji2019TheIR,Long2020Generalization} as well as closely related results that approach generalization error bounds using a PAC-Bayes framework which has been shown to yield tighter bounds \cite{langford2002pac, mcallester1999pac, dziugaite2017computing, neyshabur2017pac}. 

\section{Background and Notation}
To begin we introduce some notation and recall any relevant background material that is useful going forward. 

\subsection{Preliminaries} Let $(\Omega, \mathcal{F}, \mathbb{P})$ denote a probability space consisting of the sample space $\Omega$, the $\sigma$-algebra  $\mathcal{F}$ of subsets of $\Omega$ and the probability distribution $\mathbb{P}$ which maps sets of $\mathcal{F}$ to $[0,1]$. A random variable is a function $X$ from $\Omega$ to any set $\mathcal{S}$. In this paper, we are concerned with continuous random variables and thus will often take $\mathcal{S} = \mathbb{R}$ or $\mathbb{R}^d$ where $d>1$. In this case, we require the function $X$ to be measurable. 

The distribution of a random variable $X: \Omega \to \mathcal{S}$ is a probability measure $\mu$ on $\mathcal{S}$ defined so that $\mu(A) = \mathbb{P}(X \in A)$ for any measurable set $A \subset \mathcal{S}$. We say the probability distribution $\mu$ is continuous provided $\mu$ is absolutely continuous with respect to the Lebesgue measure $dx$ on $\mathbb{R}^d$; that is, there is a non-negative measurable function $u: \mathbb{R}^d \to [0,\infty]$ such that $d\mu = u(x) dx$.

For a real valued random variable $X$ with continuous probability distribution $\mu$, we denote the expected value of $X$ and the variance of $X$ with respect to the distribution $\mu$ as $\mathbb{E}_{\mu}[X]$ and $\Var_{\mu}(X)$, respectively. 

\paragraph{Notation} We use $\mathbbm{1}_A$ to denote the indicator function on the set $A$. We use the $\|\cdot\|_p$ to denote the $L^p$-norm on $\mathbb{R}^n$, $\|\cdot\|_F$ denotes the Frobenius norm. $\mathbb{R}_+$ denotes the positive reals.  We will use $S$ to denote a sample $\{(x_1,y_1), \dots, (x_m, y_m)\}$ of $m$ input-output pairs, at times just focusing on $S_X = \{x_1, \dots, x_m\}$ the projection over the inputs, denoted as an (input) dataset $D$. The sample $S$ is drawn i.i.d.~from some fixed but unknown data distribution $\mathscr{D}$ over the input space. In particular, we are concerned with probability data distributions which we denote by $\mu$ and their associated probability measure $\mathbb{P}$. For a set $A \subset \mathbb{R}^n$, $|A|$ denotes the Lebesgue measure of $A$ and $\mu(A)$ denotes the measure of $A$ for a measure $\mu$. Given a point $p \in \mathbb{R}^n$ and a radius $r > 0$, let $B_p(r)$ denote the ball of radius $r$ centered at $p$ and we drop the subscript $p$ to denote a ball centered at the origin; e.g. $B(r)$.

\subsection{The Poincar\'e Inequality}
Concentration inequalities are a cornerstone of convex geometry and have inspired many remarkable results and applications within numerous areas of theoretical computer science including randomized algorithms, Monte Carlo sampling methods, complexity approximation, and, of course, learning theory \cite{boucheron2004concentration, raginsky2013concentration, ledoux2001concentration}. Indeed, the concentration phenomena has been  successfully adopted or adapted across different settings and appears when studying Gaussian space, Riemannian manifolds, discrete product spaces, and algebraic structures. 

On a conceptual level, concentration inequalities quantify the amount of random fluctuations of independent random variables by bounding the probability that such a function differs from its expected value by more than a certain amount \cite{boucheron2013concentration, ledoux2001concentration}. The Poincar\'e inequality can be seen as a form of concentration, since it implies that if the function is ``spread out'' over the domain (i.e., it has a large variance), then its gradient must be correspondingly large. Conversely, if the gradients are small, then the function must be concentrated around its mean.

More precisely, the classic Poincar\'e inequality provides $L^p$ bounds on the oscillation of a function about its mean via $L^p$ bounds on its derivative. For our purposes, we focus on the case when $p=2$ though similar statement holds for all $1 \leq p \leq \infty$; see \cite{evans2022partial}. In this case, the classic Poincar\'e inequality states that, given a connected, bounded domain $U \subset \mathbb{R}^n$ with Lipschitz boundary, there exists a constant $C := C(U) > 0$ such that for every function $u: U \subset \mathbb{R}^n \to \mathbb{R}$ with $\int u^2 dx < \infty$ and $\int \|\nabla u\|^2 dx < \infty$,
\begin{equation}\label{equation:poincare_inequality}
\|u - \bar{u}\|_{L^2(U)} \leq C \|\nabla u\|_{L^2(U)},
\end{equation}
where $\bar{u} = \nicefrac{1}{|U|}\int_U u dx$ is the average value of $u$ over the domain $U$ and $|U|$ denotes the standard Lebesgue measure of $U \subset \mathbb{R}^n$. The optimal constant $C$ in \eqref{equation:poincare_inequality} is called the Poincar\'e constant and is related to the first eigenvalue of the negative Laplacian, computed via the Raleigh Quotient, and depends only on the geometry of $U$ \cite{chavel1984eigenvalues}.

The Poincar\'e inequality plays an important role in geometric and functional analysis and Poincar\'e type inequalities appear throughout the study of the geometry of metric measure spaces; i.e., metric spaces $(X,d)$ equipped with a Borel measure. In fact, it is an active area of research to better understand which properties of a metric measure space support a Poincar\'e inequality and its applications in these more abstract settings. There is a long list of metric measure spaces supporting a (local and/or non-local) Poincar\'e inequality including $\mathbb{R}^n$ \cite{evans2022partial}, Riemannian manifolds with non-negative Ricci curvature \cite{hebey2000nonlinear} and their Gromov-Hausdorff limits \cite{lott2007weak}, Carnot groups \cite{bruno2022local}, as well as other non-Riemannian metric measure spaces with fractional Hausdorff dimension \cite{heinonen2001lectures}. A more thorough overview of the literature would go out of the scope of the present paper, so we also refer the reader to \cite{heinonen2001lectures} and the survey \cite{hajlasz2000sobolev} and the references therein.

Returning to $\mathbb{R}^n$ with the standard Euclidean metric $d$, we can ask which measures $\mu$ on $(\mathbb{R}^n, d)$ satisfy a Poincar\'e type inequality as in \eqref{equation:poincare_inequality} beyond the standard Lesbesgue measure $dx$. For example, when $\mu$ is the standard Gaussian measure on $\mathbb{R}^n$, then for any smooth function $u: \mathbb{R}^n \to \mathbb{R}$, then it is known that (see, for example \cite{bakry2014analysis})
\begin{equation}\label{equation:gaussian_poincare_inequality}
\Var_{\mu} (u) \leq \mathbb{E}_{\mu}\left[|\nabla u|^2 \right]
\end{equation}
where $\Var_{\mu}(u) := \int |u - \int u d\mu|^2 d\mu$. This is known as the Gaussian Poincare Inequality and, in this setting, the Poincare constant $C=1$ is optimal. 

In fact, Poincar\'e type inequalities like \eqref{equation:poincare_inequality} and \eqref{equation:gaussian_poincare_inequality} are known to hold for wide collection of measures $\mu$ on $\mathbb{R}^n$ and take a simple form, particularly when $\mu$ is absolutely continuous with respect to the Lebesgue measure, typically referred to as  weighted Poincar\'e inequalities \cite{ferone2012remark}. In \cite{bakry2008simple} the authors prove a Poincar\'e inequality for a large class of probability measures including log-concave probability measures on $\mathbb{R}^n$. In addition, \cite{schlichting2019poincare} examines conditions on when a Poincar\'e inequality holds for mixtures of probability measures $\mu = t\mu_0 + (1-t)\mu_1$ for $t \in [0,1]$ on $\mathbb{R}^n$ and derives explicit bounds on the Poincar\'e constant for a number of useful examples including mixtures of two Gaussian measures with equal covariance matrix, mixtures of a Gaussian and sub-Gaussian measure, mixtures of two centered Gaussians with different variance, and mixtures of uniform and Gaussian measures. 

Whether or not a space supports a Poincar\'e inequality has deep connections to the geometry and analysis of the space and is closely related to the differentiability of Lipschitz functions in metric spaces \cite{cheeger1999differentiability}. Following \cite{schlichting2019poincare} and others, we define

\begin{definition}\label{definition:prob_measure_satisfies_PI}
For $n \geq 1$, a probability measure $\mu$ on $\mathbb{R}^n$ satisfies the Poincar\'e inequality with constant $\rho > 0$, if for all functions $u : \mathbb{R}^n \to \mathbb{R}$,
$$
\Var_{\mu}(u) := \int \Big|u - \int u d\mu\Big|^2 d\mu \leq \rho\int |\nabla u|^2 d\mu.
$$
In this case, we say $\mu$ satisfies $\PI(\rho)$.
\end{definition}

While the condition that a probability measure satisfies a Poincar\'e inequality is fairly ubiquitous and can be assumed for a large class of data distributions, it is possible to devise examples that are pathological in the context of machine learning and which violate this property. For example, consider a discrete distribution $\mu$ concentrated on a finite number of points $\{x_1, \dots, x_n\}$ for $x_i \in \mathbb{R}$; that is $\mu = \sum \delta_{x_i}$.  Then, a function $f$ which is constant around small neighborhoods of each $x_i$ yet for which $f(x_i) \neq f(x_j)$ for $i \neq j$, will have have zero derivative on  the support of $\mu$ but the variance of the function $f$ is non-zero.  

More generally, consider a distribution whose support consists of more than one connected components $U_1, \dots, U_n$. Again, take $f$ to be the sum of bump functions which are individually constant on a neighborhood of each $U_i$ but for which $f|_{U_i} \neq f|_{U_j}$ when $i \neq j$. Then, the gradient of $f$ is zero on each component, which implies $\mathbb{E}_{\mu}[\|\nabla f|^2] = 0$; however, the variance of $f$ is non-zero by construction. In particular, this points to a necessary condition for a distribution $\mu$ to satisfy $\PI(\rho)$. Namely, $\mu$ needs to have path connected support which is the case for most image distribution where one can morph one image smoothly into another and, more broadly, consistent with the data manifold hypothesis. 

\subsubsection{Poincar\'e and Isoperimetry}
Lastly, let us note the connection to the Poincar\'e inequality and the isoperimetric inequality, particularly recent work on the robustness of model functions via isoperimetry \cite{bubeck2023universal}. Classically, in the field of geometric analysis, the two concepts are intimately related; in fact, the isoperimetric inequality can be proven directly from the Poincar\'e inequality. In the context of machine learning and measure concentration in high-dimensional geometry, it is known that if a probability measure $\mu$ on $\mathbb{R}^n$ satisfies $\PI(\rho)$ then the tails of $\mu$ are exponentially small \cite{bobkov1997poincare}; i.e., $\mu(|x| > t) \leq C e^{-ct/\sqrt{\rho}}$, for some constants $c, C$. Essentially, the Poincar\'e inequality ensures random variables on $\mathbb{R}^n$ are $\sqrt{\rho}$-subgaussian.

As stated in \cite{bubeck2023universal}, a measure $\mu$ on $\mathbb{R}^n$ satisfies $c$-isoperimetry if for any bounded $L$-Lipschitz real-valued function satisfies
$$
\mathbb{P}[|f - \mathbb{E}_{\mu}[f]| \geq t] \leq 2e^{\frac{-nt^2}{2cL^2}},
$$

which implies that the output of any suitably rescaled Lipschitz function is in fact $O$(1)-subgaussian. In fact, as Bubeck et al.~point out, their isoperimetry assumption is related to a log-Sobolov inequality, which is a strengthening of the regular Poincar\'e inequality. This suggests that our assumption that the distribution $\mu$ of the covariates $x_i$ satisfies a Poincar\'e inequality as in our Theorem \ref{theorem:multiclass_generalization_bound} is a comparable, or even a more primitive, requirement to that of isoperimetry.

\subsection{Geometric Complexity}
We now recall the definition of the geometric complexity introduced in \cite{dherin2022neural, dherin2021geometric}:
\begin{definition}[Empirical Geometric Complexity]\label{definition:geometric_complexity}
Let $g: \mathbb{R}^d \to \mathbb{R}^k$ be a neural network. We can write $g(x) = \sigma(f(x))$ where $\sigma$ denotes the last layer activation, and $f$ its logit network.  
The \textbf{empirical geometric complexity} of the network over a dataset $D$ is defined to be the discrete Dirichlet energy of its logit network over the dataset:
\begin{equation}\label{eqn:geometric_complexity}
    \GC(f, D) = \frac{1}{|D|}\sum_{x\in D} \|\nabla_x f(x)\|_F^2,
\end{equation}
where $\|\nabla_x f(x)\|_F$ is the Frobenius norm of the network Jacobian. 
\end{definition}

In the case of simple linear transformations $f(x) = Ax + b$, the geometric complexity takes a very simple form; namely, $\GC(f, D) = \|A\|_F^2$. More generally, for Lipschitz functions, the geometric complexity is upper bounded by the square of Lipschitz constant. In fact, it is easy to construct examples for which this upper bound is strict. For example, consider the function such that $f|_{\{x \leq 0\}} = 0$, $f|_{\{x \geq 1\}} = 1$  and $f(x) = x$ otherwise. If the data set $D$ is concentrated on the set $\{x < 0\}$ or $\{x > 1\}$ then $f'(x) = 0$ so $\GC(f, D) = 0$; however, the Lipschitz constant of $f$ is 1. In the context of the generalization bound of this paper, this suggests that the comparable generalization bound proven in \cite{bartlett2017spectrally} should be looser compared to our Theorem \ref{theorem:multiclass_generalization_bound}, but we leave a more detailed comparison of the two bounds for later work.

Given a probability distribution $\mu$ on $\mathbb{R}^d$ and assuming the dataset $D$ is drawn as an i.i.d.~sample from $\mu$, the geometric complexity in Definition \ref{definition:geometric_complexity} is an unbiased estimator of the following quantity

\begin{equation}\label{equation:theoretical_geometric_complexity}
\mathbb{E}_{\mu}[\|\nabla f\|^2_F] = \int \|\nabla f(x)\|^2_F d\mu(x).
\end{equation}

We refer to the quantity defined in \eqref{equation:theoretical_geometric_complexity} as the theoretical geometric complexity as opposed to the empirical geometric complexity in Definition \ref{definition:geometric_complexity}. (Note that \cite{novak2018sensitivity} defines a similar notion of complexity using the full network rather than the logit network.)

\begin{definition}[Theoretical Geometric Complexity]\label{definition:theoretical_geometric_complexity}
Let $\mu$ denote a probability distribution on $\mathbb{R}^d$ and let $f: \mathbb{R}^d \to \mathbb{R}^k$ be the logit network of a neural network. The \textbf{theoretical geometric complexity} of $f$ with respect to $\mu$ is the expectation of $\| \nabla_x f(x)\|^2_F$ over $\mu$: 

\begin{equation}\label{eqn:theoretical_geometric_complexity}
    \GC(f, \mu) = \mathbb{E}_{\mu}[\|\nabla_x f\|^2_F].
\end{equation}
\end{definition}

Both the empirical and theoretical geometric complexity are well-defined for any differentiable model, not only a neural network, and their definitions incorporate information about both the model function and the dataset over which the task is determined. 

This being said, it is worth noting that many neural networks are not, rigorously speaking, differentiable everywhere. For example, the ReLU activation is not differentiable at the origin and thus even simple neural networks with ReLU activations are not differentiable. However, it is known by a theorem of Rademacher (see \cite{federer2014geometric}, Theorem 3.1.6) that any function $f: \mathbb{R}^m \to \mathbb{R}^n$ which is locally Lipschitz (i.e.~the restriction of $f$ to some neighborhood around any point is Lipschitz) is differentiable almost everywhere with respect to Lebesgue measure. This means that $f$ is differentiable everywhere except at most on a set of measure zero (with respect to Lebesgue). Since the theoretical geometric complexity is defined as an expectation over the distribution $\mu$, which is absolutely continuous with respect to the Lebesgue meausure, any points of non-differentiability of the model function are essentially ignored when computing $\GC(f, \mu)$. Therefore, the study and application of the geometric complexity to complex architectures remains a worthwhile endeavor.

Furthermore, this condition of differentiability may also be relaxed when interpreting the Poincar\'e inequality as in Definition \ref{definition:prob_measure_satisfies_PI} as well. There has been considerable recent research focused on extending the theory of first order calculus and differentiable functions to the realm of general metric measure spaces; see, for example, the seminal work introduced by Cheeger in \cite{cheeger1999differentiability} and later \cite{cheeger2009differentiability}. In \cite{cheeger1999differentiability}, Cheeger generalizes Rademacher's theorem to a rich class of nontrivial metric measure spaces, named \emph{Lipshitz differentiability spaces}, and shows that any doubling metric measure space satisfying a Poincar\`e inequality is a Lipschitz differentiable space. It was later shown \cite{bate2018differentiability} that, in fact, the condition of satisfying a Poincar\'e type inequality is necessary; meaning that for any metric measure space $(X, d, \mu)$ which is Lipschitz differentiable, the measure $\mu$ satisfies a Poincar\'e type inequality.

\subsubsection{Comparing the theoretical and empirical geometric complexity}\label{section:theoretical_vs_empirical}
One natural question that immediately arises is how the theoretical geometric complexity relates to (or can be bounded by) the empirical geometric complexity. In fact, we can show that for Lipschitz functions it is possible to obtain an explicit bound between the two.

Intuitively, the geometric complexity is closely related to the Lipschitz constant of a function though they are inherently different.  Recall, the Lipschitz constant of a map $f: \mathbb{R}^d \to \mathbb{R}^k$ is the smallest $L > 0$ such that $\|f(x_1) - f(x_2)\| \leq L \|x_1 - x_2\|$ for all $x_1, x_2 \in \mathbb{R}^d$.  Conceptually, the constant $L$ measures the maximal amount of variation allowed by the function $f$ over its entire domain provided the inputs change by a fixed amount. The geometric complexity measures intrinsic stretching of the function, more closely related to the Dirichlet energy, and is either averaged over a dataset (as in the empirical geometric complexity) or over the entire data distribution (as in the theoretical geometric complexity). We show

\begin{proposition}\label{proposition:relate_theoretical_empirical_gc}
Let $f: \mathbb{R}^d \to \mathbb{R}^k$ be an $L$-Lipschitz map and $D = \{x_i\}_{i=1}^m$ a dataset of $m\geq 1$ points $x_i \in \mathbb{R}^d$ sampled from the continuous probability distribution $\mu$. Then, for any $\delta \geq 0$, with probability at least $1 - \delta/2$, the following holds:
$$
\GC(f, \mu) \leq \GC(f, D) + L \sqrt{\dfrac{\log \frac{2}{\delta}}{2m}}.
$$
\end{proposition}
\begin{proof}
See proof in Section \ref{section:appendix_theoretical_vs_empirical} of the Appendix.
\end{proof}

\subsection{Covering Numbers}
The covering number for a class of functions plays an important role in Learning Theory and can used to estimate the probability or number of samples required to obtain a given confidence and error bound \cite{mendelson2001geometric}. In short, covering numbers provide a measure of the complexity of a class of functions; the larger the covering number the more complex the set of functions is. 

In general, given a metric space $(X, d)$ and a subset $U \subset X$, the \textbf{$\epsilon$-covering number} of $U$, denoted $\mathcal{N}(U, \epsilon, d)$ is the minimal number of $\epsilon$-balls (with respect to the metric $d$) needed to cover $U$. A collection $\mathcal{C}_{\epsilon}$ of elements $u_1, \dots, u_k \in U$ is said to be an \textbf{$\epsilon$-covering} of $U$ if the union of balls $\cup_i B_{u_i}(\epsilon)$ contains $U$; that is, any $u \in U$, there exists $i \in [k]$ such that $d(u, u_i) \leq \epsilon$. 

In this paper we are concerned with the class of differentiable functions $\mathcal{F} \ni f: \mathbb{R}^d \to \mathbb{R}^k$. Given a sample 
$S = \{(x_1, y_1),  \dots, (x_m, y_m)\}$, let $S_X = \{x_1, \dots, x_m\}$ denote the projection over the inputs of $S$. We can also represent the input $x_i$'s as a data matrix $X \in \mathbb{R}^{m\times d}$ by collecting the examples as rows of the matrix $X$.  Denote by $\mu_{S_X} = m^{-1}\sum \delta_{\{x_i\}}$ the empirical measure supported on $S_X$ and endow $\mathbb{R}^d$ with the Euclidean structure of $L^2(\mu_{S_X})$, which is isometric to $\ell^2$, and let $\left.\mathcal{F}\right|_{S} = \{\sum_i^mf(x_i)\mathbbm{1}_{x_i} ~|~f \in \mathcal{F}\} \subset L^2(\mu_{S_X})$.

For a given function $f \in \mathcal{F}$ we will  consider the image $f(X) = \{f(x) ~|~ x \in X \} \subset \mathbb{R}^{k}$ of the set $S_X$ and derive a covering number estimate on $\left.\mathcal{F}\right|_{S}$; note that 
$\left.\mathcal{F}\right|_{S} = \{f(X) ~|~ f \in \mathcal{F}\} = \{f(x) ~|~ f \in \mathcal{F}, x \in X\}$. 

It is possible to show (see \cite{dudley1967sizes, tomczak1989banach} and the discussion in \cite{mendelson2001geometric}), that the empirical Rademacher complexity of $\left.\mathcal{F}\right|_{S}$ can be upper bounded in terms in the covering numbers of $\mathcal{F}$ in $L^2(\mu_{S_X})$. 
One of the key components of our proof is an application of a variant of this Dudley entropy integral as utilized in \cite{bartlett2017spectrally}.

\section{Main Results}
In this paper, we are concerned with a multiclass generalization bound for neural networks. We consider a collection of data points $S = \{(x_1, y_1),  \dots, (x_m, y_m)\}$ sampled from a probability distribution over $\mathbb{R}^d \times \{1, \dots, k\}$. Given a model function $f: \mathbb{R}^d \to \mathbb{R}^k$, for any input example $x \in \mathbb{R}^d$, the network output $f(x) \in \mathbb{R}^k$ is converted to a class label $\{ 1, \dots, k\}$. For the sake of clarity and simplicity, we will focus on the case when $k=1$ and $k>1$ separately though the argument follows the same idea; the details and notation requires just a bit more attention in the $k > 1$ setting.

In general, the argument to prove our generalization bound via geometric compleixty has three steps. \textbf{Step 1:} Assuming the Poincar\'e inequality for the input data distribution, we can show that the covering number of the class of functions with bounded $\GC$ is also bounded by a term involving the $\GC$. \textbf{Step 2:} This bound on the covering number then allows us to derive a bound on the Rademacher complexity of that function class following a classical argument using the Dudley entropy integral. \textbf{Step 3:} We can at this point use the standard generalization bound in terms of the Rademacher complexity to derive our final generalization bound w.r.t.~the $\GC$. 

\subsection{The case  $k=1$}
Let's work out carefully the binary classification case; i.e., when $k=1$ and the neural network maps to $\mathbb{R}$. 

Following the proof outline above, the following Lemma accomplishes Step 1. We include the full proof here to illustrate the main application of our concentration inequalities and how exactly they are exploiting given a bound on the geometric complexity of the functions in the class.

\begin{lemma}\label{lemma:PI_covering_number}
Let $\mu$ be a probability measure on $\mathbb{R}^d$ that satisfies a Poincar\'e inequality with constant $\rho > 0$. Denote by $X \in \mathbb{R}^{m \times d}$ a data matrix consisting of $m$ points $x_i \in \mathbb{R}^d$ for $i \in [m]$. Let $a_1, a_2, \epsilon$ be positive reals and let $\mathcal{F}$ denote the class of differentiable maps
$$
\mathcal{F} := \{f: \mathbb{R}^d \to \mathbb{R} ~|~ \GC(f, \mu) \leq a_1, |\mathbb{E}_{\mu}[f]| \leq a_2\}.
$$

Then, for any $\delta > 0$, with probability at least $1 - \delta$,
$$
\mathcal{N}\left(\left\{f(X) ~|~ f \in \mathcal{F} \right)\}, \epsilon, \|\cdot\|_2\right) 
\leq 
\dfrac{1}{\epsilon}\left(a_2 + \sqrt{\dfrac{a_1 \rho}{\delta}}\right).
$$
\end{lemma}

\begin{proof}
Given $f \in \mathcal{F}$, define a new function $\tilde{f}(x) = f(x) - \mathbb{E}_{\mu}[f]$. Then $\mathbb{E}_{\mu}[\tilde{f}] = 0$ and $\GC(\tilde{f}, \mu) = \GC(f, \mu) \leq a_1$. By Chebyshev's inequality and since $\mu$ satisfies $\PI(\rho)$, for any $t \in \mathbb{R}_+$, we have
$$
\mathbb{P}\big(|\tilde{f}|\leq t \big) \geq  1 - \dfrac{\Var_{\mu}(\tilde{f})}{t^2} \geq 1 - \dfrac{\rho \GC(\tilde{f}, \mu)}{t^2} \geq 1- \dfrac{ a_1\rho}{t^2}.
$$

Taking $\delta = \dfrac{a_1\rho}{t^2}$ and solving for $t$ we get $t = \sqrt{a_1\rho/\delta}$. Thus it follows that, for $\delta \in (0,1)$,
$$\mathbb{P}\Big(|\tilde{f}| \leq \sqrt{a_1\rho/\delta}\Big) \geq 1-\delta.$$

Therefore, 
$$
\mathbb{P}\Big(|f - \mathbb{E}_{\mu}[f]| \leq \sqrt{a_1\rho/\delta}\Big) \geq 1-\delta,
$$
and since $|\mathbb{E}_{\mu}[f]| \leq a_2$, we have
$$
\mathbb{P}\Big(|f| \leq a_2 + \sqrt{a_1\rho/\delta}\Big) \geq 1-\delta.
$$

By the above argument, for any $f \in \mathcal{F}$ and for any $x \in \mathbb{R}^d$, with probability $1 - \delta$, $f(x)$ is contained inside the interval $\left[-a_2 - \sqrt{a_1\rho/\delta}, a_2 + \sqrt{a_1\rho/\delta}\right]$. Set $L := a_2 + \sqrt{a_1\rho/\delta}$; therefore,   $\{f(x)~|~f\in \mathcal{F}, x \in X\} \subset [-L , L]$.

To find the size of an $\epsilon$-covering of the interval $[-L, L]$, one can divide the interval into intervals of length $2\epsilon$. This bounds the size of the $\epsilon$-covering of the convex hull of $f(X) = \{f(x) ~|~ x \in X\}$ of points $f(x_i) \in \mathbb{R}$ by $L/\epsilon$. Thus, we get
$$
\mathcal{N}(\{f(X) ~|~ f \in \mathcal{F}\}, \epsilon, \|\cdot\|_2) \leq L/\epsilon = \dfrac{1}{\epsilon}\left(a_2 + \sqrt{\dfrac{a_1 \rho}{\delta}}\right).
$$

\end{proof}

\subsubsection{Confidence margin analysis}
The confidence margin of a real-valued function $f$ at a data point $(x, y)$ is given by $yf(x)$ where we interpret the magnitude $|f(x)|$ as the confidence of the prediction when $f$ classifies $x$ correctly; i.e., $yf(x) > 0$. For any $\gamma >0$, the $\gamma$-margin loss function, or ramp loss, penalizes $f$ both when it misclassifies a point and when it correctly classifies $x$ with confidence less than $\gamma$; i.e. $yf(x) \leq \gamma$. We define the ramp loss $\ell_{\gamma} : \mathbb{R} \to \mathbb{R}_+$ as 
$$
\ell_{\gamma}(r) := \begin{cases}
0 & r < -\gamma,\\
1 + r/\gamma & r \in [-\gamma, 0],\\
1 & r >0,
\end{cases}
$$
The parameter $\gamma$ is the confidence margin demanded from a hypothesis $f$. Given a sample $S = \{(x_1, y_1), \dots, (x_m, y_m)\}$ the empirical margin loss is defined by 
$$
\widehat{\mathcal{R}}_{S, \gamma}(f) := m^{-1}\sum_{i=1}^m \ell_{\gamma}(-y_if(x_i)).
$$

We arrive at the following generalization bound:
\begin{theorem}\label{theorem:binary_generalization_bound}
Given $a_1, a_2$ be positive reals. Let $S = \{(x_1,  y_1), \dots, (x_m, y_m)\}$ be i.i.d.~input-output pairs in $\mathbb{R}^d \times \{\pm 1\}$ and suppose the distribution $\mu$ of the $x_i$ satisfies the Poincar\'e inequality with constant $\rho >0$.  Then, for any $\delta > 0$, with probability at least $1 - \delta$, every margin $\gamma > 0$ and network $f: \mathbb{R}^d \to \mathbb{R}$ which satisfies $\GC(f, \mu) \leq a_1$ and $|\mathbb{E}_{\mu}(f)| \leq a_2$
$$
\mathbb{P}\left[yf(x) \leq 0\right] \leq \widehat{\mathcal{R}}_{S, \gamma}(f) + 
 \dfrac{12\tilde{C}\sqrt{\pi}}{\gamma m} + 
 3\sqrt{\dfrac{\log \frac{2}{\delta}}{2m}}
$$
where $\widehat{\mathcal{R}}_{S, \gamma}(f) = m^{-1} \sum_i \mathbbm{1}[y_if(x_i) \leq \gamma]$ and $\tilde{C} = a_2 + \sqrt{a_1\rho/\delta}$.
\end{theorem}

See Appendix \ref{section:appendix_binary_proof} for the complete proof of Theorem \ref{theorem:binary_generalization_bound}.

\subsection{The case when $k > 1$}
Most of the argument for the case $k>1$ follows as above. However, in the multi-class setting the margin-based bound takes a slightly modified form c.f.~\cite{bartlett2017spectrally, mohri2018foundations}. In the multi-class setting the label associated to a data point $x$ is the one resulting in the highest score from the model. This leads to the definition of the margin operator $\mathcal{M}: \mathbb{R}^k \times \{1, \dots, k\} \to \mathbb{R}$ defined as $\mathcal{M}(v,y) := v_y - \max_{i \neq y} v_i$ where $v_i$ denotes the $i$-th component of the vector $v$. Thus, given a funciton $f: \mathbb{R}^d \to \mathbb{R}^k$, the margin of the function at a labeled example $(x, y)$ is 
$$
\mathcal{M}(f(x), y) = f(x)_y - \max_{i \neq y} f(x)_i
$$
and $f$ misclassifies the datapoint $(x,y)$ iff $\mathcal{M}(f(x), y) \leq 0$. Similarly to the binary classification case $k=1$, given a sample $S = \{(x_1, y_1), \dots, (x_m, y_m)\}$ and a confidence margin $\gamma >0$, we can define the empirical margin loss of $f$ for multi-class classification as 
$$
\widehat{\mathcal{R}}_{S, \gamma}(f) := m^{-1} \sum_{i=1}^m \ell_{\gamma}(-\mathcal{M}(f(x_i), y_i)).
$$
Thus, the empirical margin loss $\widehat{\mathcal{R}}_{S, \gamma}(f)$ is upper bounded by the fraction of training examples that are either misclassified by $f$ or that are correctly classified but with confidence below the threshold $\gamma$.

Similarly to Lemma \ref{lemma:PI_covering_number}, we have

\begin{lemma}\label{lemma:PI_covering_number_multiclass}
Let $\mu$ be a probability measure on $\mathbb{R}^d$ that satisfies a multi-dimensional Poincar\'e inequality with constant $\rho > 0$ for maps $f: \mathbb{R}^d \to \mathbb{R}^k$, for $k>1$. Denote by $X \in \mathbb{R}^{m \times d}$ a data matrix consisting of points $x_i \in \mathbb{R}^d$ for $i \in [m]$. Given positive reals $a_1, a_2$, and $\epsilon$ , let $\mathcal{F}$ denote the class of differentiable maps
$$
\mathcal{F} := \{f: \mathbb{R}^d \to \mathbb{R}^k ~|~ \GC(f, \mu) \leq a_1, \|\mathbb{E}_{\mu}[f]\| \leq a_2\}.
$$

Then, for any $\delta > 0$, with probability at least $1 - \delta$,
$$
\mathcal{N}(\{f(X) ~|~ f \in \mathcal{F}\}, \epsilon, \|\cdot\|_2) \leq \dfrac{3^k}{\epsilon^k}\left(a_2 + \sqrt{\dfrac{a_1 \rho}{\delta}}\right)^k.
$$
\end{lemma}

The proof follows a similar idea behind the proof of Lemma \ref{lemma:PI_covering_number}, taking care to adapt the argument to the setting of $\mathbb{R}^k$ and is contained in Appendix \ref{section:appendix_multiclass_proof}. 

With this Lemma we can now prove our multi-class margin based  generalization bound, Theorem \ref{theorem:multiclass_generalization_bound}. The full proof is detailed in the Appendix \ref{section:appendix_multiclass_proof}.

\section{Discussion and Conclusion}
In this paper we investigate margin-based multiclass generalization bounds for neural networks and show that the generalization error can be upper bounded by the geometric complexity of the learned model function. Indeed, we see that lower geometric complexity  implies tighter generalization gaps. Furthermore, these upper bounds on the generalization error scale with the margin-normalized geometric complexity of the network.

In our experiments we train a ResNet-18 model on the CIFAR-10 and CIFAR-100 datasets with original and random labels and see that indeed the excess risk during training is closely correlated with the geometric complexity of the model function in both settings, meaning they the two quantities tend to increase at a comparable rate during training 

We also show how the empirical geometric complexity (computed over a data sample) compares to the theoretical geometric complexity (measured over the entire data distribution). Namely we show that for Lipschitz functions these two quantities are comparable.

Our results hold for a broad family of data distributions and model classes. Namely, we require that the distributions from which our data is sampled satisfies a Poincar\'e inequality. The Poincar\'e inequality plays an important role in geometric and functional analysis and Poincar\'e type inequalities appear throughout the study of the geometry of metric measure spaces. Furthermore, due to the nature of the definition of the geometric complexity, our bounds have no dependence on artifacts of the network architecture such as number or depth of layers.

We believe that this work gives a interesting and useful application of a recently proposed complexity measure for neural networks. Furthermore, our approach and proof techniques highlight the role and importance of the geometry of the underlying data distribution. We hope
that our work will inspire further research in this area.

\vspace{-.1in}
\section{Acknowledgements}
The authors would like to thank thank Scott Yak, Mehryar Mori and Peter Bartlett for many helpful discussions during the course of this work, as well as the reviewers for thoughtful and helpful feedback, which greatly improved the exposition.
\bibliography{example_paper}
\bibliographystyle{icml2023}

\newpage
\appendix
\onecolumn
\section{Comparing theoretical and geometric complexity}\label{section:appendix_theoretical_vs_empirical}
\begin{proof}[Proof of Proposition \ref{proposition:relate_theoretical_empirical_gc}]
For any dataset $D= \{x_i\}_{i=1}^m$ of $m \geq 1$ points $x_i \in \mathbb{R}^d$ drawn as i.i.d.~samples from the continuous probability distribution $\mu$ over $\mathbb{R}^d$, the empirical geometric complexity over $D$ is denoted by $\GC(f, D)$. We start by showing that 
$$
\mathbb{E}_{D \sim \mu^m}\left[\GC(f, D)\right] = \GC(f, \mu).
$$
In fact, this follows by computation, keeping in mind that $\mu$ is a probability distribution and that the points are independently  sampled. Note that,
\begin{eqnarray*}
\mathbb{E}_{D \sim \mu^m}\left[\GC(f, D)\right] 
&=& \mathbb{E}_{x_1, \dots, x_m \sim \mu^m}\left[\dfrac{1}{m}\sum_{i=1}^m \|\nabla_x f(x_i)\|^2_F \right] \\
&=& \dfrac{1}{m} \int_{\mathbb{R}^{m \times d}}\sum_{i=1}^m \|\nabla_x f(x_i)\|^2_F d\mu^m(x_1,\dots,x_m) \\
&=& \dfrac{1}{m} \int_{\mathbb{R}^{m \times d}}\sum_{i=1}^m \|\nabla_x f(x_i)\|^2_F u(x_1)\cdots u(x_m) dx_1 \cdots dx_m \\
&=& \dfrac{1}{m} \sum_{i=1}^m \int_{\mathbb{R}^{(m-1) \times d}} \left[\int_{\mathbb{R}^d} \|\nabla_x f(x_i)\|^2_Fu(x_i)dx_i \right] u(x_1) \cdots\widehat{u(x_i)}\cdots u(x_m) dx_1 \cdots \widehat{dx_i}\cdots dx_m \\
&=& \dfrac{1}{m} \sum_{i=1}^m  \left[\int_{\mathbb{R}^d} \|\nabla_x f(x_i)\|^2_Fu(x_i)dx_i \right] \\
&=& \dfrac{1}{m} \sum_{i=1}^m  \left[\int_{\mathbb{R}^d} \|\nabla_x f(x_i)\|^2_Fd\mu(x_i) \right] \\
&=& \dfrac{1}{m} \sum_{i=1}^m  \GC(f, \mu) \\
&=& \GC(f, \mu).
\end{eqnarray*}

Let $D$ and $D'$ be two samples of size $m \geq 1$ which differ by exactly one point, say $x_i$ in $D$ and $x_i'$ in $D'$. Then since the map $f$ is $L$-Lipschitz we have

$$
\GC(f, D) - \GC(f, D') = \dfrac{1}{m}\left(\|\nabla_x f(x_i)\|^2_F - \|\nabla_x f(x_i')\|^2_F\right) \leq L^2/m,
$$
and similarly, $\GC(f, D') - \GC(f, D) \leq L^2/m$. Thus,
$
|\GC(f, D) - \GC(f, D')| \leq L^2/m
$
and by applying McDiarmind's inequality (e.g. \cite{mohri2018foundations}), we have that for any $\epsilon > 0$, 
\begin{equation}\label{equation:mcdiarmids}
\mathbb{P}\left[\GC(f, D) - \mathbb{E}_{D \sim \mu^m}[\GC(f, D)] \leq \epsilon\right] \geq 1 - \exp(-2m\epsilon^2/L^2).
\end{equation}

Thus, since $\mathbb{E}_{D \sim \mu^m}[\GC(f, D)] = \GC(f, \mu)$ and setting $\delta/2 = \exp(-2m\epsilon^2/L^2)$ and substituting for $\epsilon$ in \eqref{equation:mcdiarmids}, we get that for any $\delta > 0$ with probability as least $1 - \delta/2$ the following holds:
$$
\GC(f, \mu) \leq \GC(f, D) + L \sqrt{\dfrac{\log \frac{2}{\delta}}{2m}}.
$$
This completes the proof. 
\end{proof}

\section{Proof of Theorem \ref{theorem:binary_generalization_bound}}\label{section:appendix_binary_proof}
Let us restate the theorem and provide the proof:

\begin{theorem}
Given $a_1, a_2$ be positive reals. Let $S = \{(x_1,  y_1), \dots, (x_m, y_m)\}$ be i.i.d.~input-output pairs in $\mathbb{R}^d \times \{\pm 1\}$ and suppose the distribution $\mu$ of the $x_i$ satisfies the Poincar\'e inequality with constant $\rho >0$.  Then, for any $\delta > 0$, with probability at least $1 - \delta$, every margin $\gamma > 0$ and network $f: \mathbb{R}^d \to \mathbb{R}$ which satisfies $\GC(f, \mu) \leq a_1$ and $|\mathbb{E}_{\mu}(f)| \leq a_2$
$$
\mathbb{P}\left[yf(x) \leq 0\right]  
\leq \widehat{\mathcal{R}}_{S, \gamma}(f) + \dfrac{12\tilde{C}\sqrt{\pi}}{\gamma m} + 3\sqrt{\dfrac{\log \frac{2}{\delta}}{2m}}
$$
 where $\widehat{\mathcal{R}}_{S, \gamma}(f) = m^{-1} \sum_i \mathbbm{1}_{y_if(x_i) \leq \gamma}$ and $\tilde{C} = a_2 + \sqrt{a_1\rho/\delta}$.
\end{theorem}
The proof follows by combining fairly standard arguments in the literature. We include the full details here for completeness.
\begin{proof}
Let $\mathcal{F}$ denote the class of differentiable maps
$$
\mathcal{F} := \{f: \mathbb{R}^d \to \mathbb{R} ~|~ \GC(f, \mu) \leq a_1, |\mathbb{E}_{\mu}[f]| \leq a_2\}.
$$
and let $\widetilde{\mathcal{F}} = \{z = (x,y) \mapsto yf(x) ~|~ f \in \mathcal{F}\}$. For any $\gamma >0$, define 
$$
\widetilde{\mathcal{F}}_{\gamma} := \{(x,y) \mapsto \ell_{\gamma}(-yf(x)) ~|~ f \in \mathcal{F}\}.
$$

Since $\ell_{\gamma}$ has range $[0,1]$, it follows classic generalization bounds based on the Rademacher complexity (see, for example Theorem 3.3 in \cite{mohri2018foundations}) that, for any $\delta > 0$, with probability at least $1 - \delta$ over the draw of an i.i.d.~sample $S$ of size $m$, we have for all $f \in \widetilde{\mathcal{F}}_{\gamma}$:
\begin{equation}\label{equation:Rademacher_generalization}
\mathbb{E}[\ell_{\gamma}(-yf(x))] 
\leq 
\dfrac{1}{n}\sum_{i=1}^n \ell_{\gamma}(-y_if(x_i)) + 
  2 \widehat{\mathfrak{R}}_S(\widetilde{\mathcal{F}}_{\gamma}) +   3\sqrt{\dfrac{\log \frac{2}{\delta}}{2m}}.
\end{equation}

We can further simplify the term $\widehat{\mathfrak{R}}_S(\widetilde{\mathcal{F}}_{\gamma})$ here. Namely, $\widehat{\mathfrak{R}}_S(\widetilde{\mathcal{F}}_{\gamma}) = \widehat{\mathfrak{R}}_S(\ell_{\gamma} \circ \widetilde{\mathcal{F}})$ and, since the ramp loss $\ell_{\gamma}$ is $1/\gamma$-Lipschitz, by Talagrand's lemma (e.g.~see Lemma 5.7 of \cite{mohri2018foundations}), the empirical Rademacher complexity of $\ell_{\gamma} \circ \widetilde{\mathcal{F}}$ can be bounded in terms of the empirical Rademacher complexity of the original hypothesis set $\widetilde{\mathcal{F}}$; that is,
\begin{equation}\label{equation:Talagrand_lipschitz}
\widehat{\mathfrak{R}}_S(\ell_{\gamma} \circ \widetilde{\mathcal{F}}_{\gamma}) \leq \dfrac{1}{\gamma}\widehat{\mathfrak{R}}_S(\widetilde{\mathcal{F}}).
\end{equation}

Since the $y_i \in \{\pm 1\}$, by computing the empirical Rademacher complexity of $\widetilde{\mathcal{F}}$ over the set $S$, we also have $\widehat{\mathfrak{R}}_S(\widetilde{\mathcal{F}}) = \widehat{\mathfrak{R}}_S(\mathcal{F})$. Therefore, and by recalling the definition of $\widehat{\mathcal{R}}_{S, \gamma}(f)$, \eqref{equation:Rademacher_generalization} becomes
$$
\mathbb{E}_{\mu}[\ell_{\gamma}(-yf(x))] 
\leq \widehat{\mathcal{R}}_{S, \gamma}(f) + \dfrac{2}{\gamma}\widehat{\mathfrak{R}}_S(\mathcal{F}) + 3\sqrt{\dfrac{\log \frac{2}{\delta}}{2m}}.
$$

Focusing now on the left hand side of \eqref{equation:Rademacher_generalization}, note that by definition of the ramp loss, since $\mathbbm{1}_{-yf(x) \geq 0} \leq \ell_{\gamma}(-yf(x))$, we have
$$
\mathbb{E}[\mathbbm{1}_{-yf(x) \geq 0}] 
\leq 
\mathbb{E}[\ell_{\gamma}(-yf(x))]
$$
and $\mathbb{P}\left[yf(x) \leq 0\right] = \mathbb{E}[\mathbbm{1}_{-yf(x) \geq 0}]$. Therefore, 
$$
\mathbb{P}\left[yf(x) \leq 0\right]
\leq \widehat{\mathcal{R}}_{S, \gamma}(f) + \dfrac{2}{\gamma}\widehat{\mathfrak{R}}_S(\mathcal{F}) + 3\sqrt{\dfrac{\log \frac{2}{\delta}}{2m}}.
$$

Furthermore, by definition of the ramp loss, we have that 
\begin{eqnarray*}
\mathbb{P}\left[yf(x) \leq 0\right] 
&=& \mu(-yf(x) \geq 0) \\
&=& \mathbb{E}_{\mu}[\mathbbm{1}_{-yf(x) \geq 0}]\\
&\leq& \mathbb{E}_{\mu}[\ell_{\gamma}(-yf(x))].
\end{eqnarray*}

Therefore, 
\begin{equation}\label{equation:Rademacher_generalization_simplified}
\mathbb{P}\left[yf(x) \leq 0\right]  
\leq \widehat{\mathcal{R}}_{S, \gamma}(f) + \dfrac{2}{\gamma}\widehat{\mathfrak{R}}_S(\mathcal{F}) + 3\sqrt{\dfrac{\log \frac{2}{\delta}}{2m}}.
\end{equation}

To complete the proof we can use a form of the Dudley entropy integral to deduce an upper bound on $\widehat{\mathfrak{R}}_S(\mathcal{F})$. The Dudley entropy integral lemma (see Lemma A.5 of  \cite{bartlett2017spectrally}) states that 
$$
\widehat{\mathfrak{R}}_S(\mathcal{F}) \leq \inf_{\alpha > 0}\left(\dfrac{4\alpha}{\sqrt{m}} +  \dfrac{12}{m}\int_{\alpha}^{\sqrt{m}} \sqrt{\log \mathcal{N}\left(\left.\mathcal{F}\right|_S, \epsilon, \|\cdot\|_2 \right)} d\epsilon\right) .
$$
Examining the integral term above, note that $\left.\mathcal{F}\right|_S = \{f(X) ~|~f \in \mathcal{F}\} = \{f(x) ~|~f\in\mathcal{F}, x \in X\}$ where $X = S_X$ is the projection of the sample $S$ onto the inputs, so $\mathcal{N}(\mathcal{F}|_S, \epsilon, \|\cdot \|_2) = \mathcal{N}(\{f(X) ~|~f\in \mathcal{F}\}, \epsilon, \|\cdot \|_2)$ and, as in Lemma \ref{lemma:PI_covering_number} taking $\tilde{C} = \tilde{C}(a_1, a_2, \rho, \delta) := a_2 + \sqrt{a_1 \rho/\delta}$, it follows that   $\mathcal{N}(\{f(X) ~|~f\in \mathcal{F}\}, \epsilon, \|\cdot \|_2) = 1$ for all $\epsilon \geq \tilde{C}$ since  it requires only one ball of radius greater than or equal to $\tilde{C}$ to cover a ball of radius $\tilde{C}$. Thus,  the integrand above is zero for any $\epsilon \geq \tilde{C}$. We can further upper bound this integral by swapping the integral limit $\sqrt{m}$ with $\tilde{C}$ since the integral of a positive function is no greater than the integral of that function over a potentially larger domain. 
Therefore, we get, 
\begin{eqnarray*}
\widehat{\mathfrak{R}}_S(\mathcal{F})
&\leq& 
\inf_{\alpha > 0} \left\{ \dfrac{4\alpha}{\sqrt{m}} + \dfrac{12}{m} \int_{\alpha}^{\min(\sqrt{m}, \tilde{C})}  \sqrt{\log \mathcal{N}(\mathcal{F|_S}, \epsilon, \|\cdot \|_2)} d\epsilon \right\} \\
&\leq&
\inf_{\alpha > 0} \left\{ \dfrac{4\alpha}{\sqrt{m}} + \dfrac{12}{m} \int_{\alpha}^{\tilde{C}}  \sqrt{\log \mathcal{N}(\mathcal{F|_S}, \epsilon, \|\cdot \|_2)} d\epsilon \right\}
\end{eqnarray*}
To simplify this, let's first compute the integral term. By Lemma \ref{lemma:PI_covering_number},
\begin{eqnarray*}
\int_{\alpha}^{\tilde{C}} \sqrt{\log \mathcal{N}\left(\left.\mathcal{F}\right|_S, \epsilon, \|\cdot\|_2 \right)} d\epsilon 
&\leq& \int_{\alpha}^{\tilde{C}} \sqrt{\log \big(\tilde{C}/\epsilon\big)} d\epsilon \\
&=& \left.\epsilon \sqrt{\log \big(\tilde{C}/\epsilon\big)} \right|_{\epsilon = \alpha}^{\epsilon = \tilde{C}} - \left.\dfrac{\tilde{C}\sqrt{\pi}}{2}\erf\left(\sqrt{\log \big(\tilde{C}/\epsilon\big)}\right)\right|_{\epsilon = \alpha}^{\epsilon = \tilde{C}}
\end{eqnarray*}
where $\erf$ denotes the error function
$$
\erf(z) = \dfrac{2}{\sqrt{\pi}}\int_0^z e^{-t^2} dt.
$$
Evaluating the right hand side fully, we get
\begin{eqnarray*}
\int_{\alpha}^{\tilde{C}} \sqrt{\log \mathcal{N}\left(\left.\mathcal{F}\right|_S, \epsilon, \|\cdot\|_2 \right)} d\epsilon 
&\leq& \left.\epsilon \sqrt{\log \big(\tilde{C}/\epsilon\big)} \right|_{\epsilon = \alpha}^{\epsilon = \tilde{C}} - \left.\dfrac{\tilde{C}\sqrt{\pi}}{2}\erf\left(\sqrt{\log \big(\tilde{C}/\epsilon\big)}\right)\right|_{\epsilon = \alpha}^{\epsilon = \tilde{C}} \\
&=& -\alpha \sqrt{\log(\tilde{C}/\alpha)} + \dfrac{\tilde{C}\sqrt{\pi}}{2}\erf\left(\sqrt{\log \big(\tilde{C}/\alpha\big)}\right) \\
&\leq& \dfrac{\tilde{C}\sqrt{\pi}}{2} - \alpha \sqrt{\log(\tilde{C}/\alpha)},
\end{eqnarray*}
where in the last inequality we simply used the fact that for any $z >0$ we have $\erf(z) \leq 1$. Therefore, substituting this back into the entropy bound for $\widehat{\mathfrak{R}}_S(\mathcal{F})$ above, and bounding the $\inf$ by taking the limit  $\alpha$ goes to zero; we get,
\begin{eqnarray*}
\widehat{\mathfrak{R}}_S(\mathcal{F})
&\leq&
\inf_{\alpha > 0} \left\{ \dfrac{4\alpha}{\sqrt{m}} + \dfrac{12}{m}\left(\dfrac{\tilde{C}\sqrt{\pi}}{2} - \alpha \sqrt{\log(\tilde{C}/\alpha)} \right)\right\} \\
&\leq& \lim_{\alpha \to 0} \left\{ \dfrac{4\alpha}{\sqrt{m}} + \dfrac{12}{m}\left(\dfrac{\tilde{C}\sqrt{\pi}}{2} - \alpha \sqrt{\log(\tilde{C}/\alpha)} \right)\right\}\\
&=& \dfrac{6\tilde{C}\sqrt{\pi}}{m} 
\end{eqnarray*}

Note that in the inequalities above we are not finding the optimal or tightest upper bounds for $\widehat{\mathfrak{R}}_S(\mathcal{F})$ that are possible. However, given the nature of these expressions it is possible to determine bounds on how sharp these inequalities are. We simply note for the time being that, although these bounds are not sharp, they are not gross overestimates of the true infimum. Finally, substituting this bound on $\widehat{\mathfrak{R}}_S(\mathcal{F})$ into \eqref{equation:Rademacher_generalization_simplified} we get.
\begin{equation*}
\mathbb{P}\left[yf(x) \leq 0\right]  
\leq \widehat{\mathcal{R}}_{S, \gamma}(f) + \dfrac{12\tilde{C}\sqrt{\pi}}{\gamma m} + 3\sqrt{\dfrac{\log \frac{2}{\delta}}{2m}}
\end{equation*}
which completes the proof.

\end{proof}

\section{Proof of Theorem \ref{theorem:multiclass_generalization_bound}}\label{section:appendix_multiclass_proof}

One of the key components of our proof is the Poincar\'e inequality, originally stated for real-valued functions as in \cite{evans2022partial}. Under similar assumption the Poincar\'e inequality naturally extends to vector valued maps. We include the proof here.

Let's now detail the proof of the main covering lemma behind Theorem \ref{theorem:multiclass_generalization_bound}. As mentioned previously, the proof follows the same logic as the idea as case $k=1$ only here we need to be a bit more careful about multivariate norms. Note also, the final ball counting argument on the image in $\mathbb{R}^k$ incurs an additional cost resulting in an exponent $k$ which ultimately incurs a cost of a factor $\sqrt{k}$ in our final bound; c.f. \cite{zhang2004statistical}.

\begin{proof}[Proof of Lemma \ref{lemma:PI_covering_number_multiclass}]
Given $f\in\mathcal{F}$, let $f^i$ denote the component functions of $f$ for $i \in [k]$ and define $\tilde{f}:\mathbb{R}^d \to \mathbb{R}^k$ by
$$
\tilde{f} := (f^1 - \mathbb{E}_{\mu}[f^1], \dots, f^k - \mathbb{E}_{\mu}[f^k]).
$$
Thus, $\mathbb{E}_{\mu}[\tilde{f}] = 0 \in \mathbb{R}^k$ and $\GC(\tilde{f}, \mu) = \GC(f, \mu) \leq a_1$. Futhermore, by extending Chebyshev's inequality to this multivariate setting, we get that, for any $t \in \mathbb{R}_+$,
$$
\mathbb{P}\left[\|\tilde{f}\| \leq t \right] \geq 1 - \dfrac{\sum_i\Var_{\mu}(\tilde{f}^i)}{t^2}.
$$
Note that by the definition of $\tilde{f}$ and since $\mu$ satisfies $\PI(\rho)$, for each $i \in [k]$,
$$
\Var_{\mu}(\tilde{f}^i) = \int |\tilde{f}^i|^2 d\mu \leq \rho \int \|\nabla \tilde{f}^i\|^2 d\mu = \rho \GC(\tilde{f}^i, \mu).
$$
Furthermore, by the definition of $\GC(f, \mu)$ it follows that 
$$
\GC(\tilde{f}, \mu) = \int \|\nabla_x \tilde{f}\|^2_F d\mu = \int \sum_{i,j} \left| \dfrac{\partial \tilde{f}^i}{\partial x_j}\right|^2 d\mu = \sum_i \GC(\tilde{f}^i, \mu). 
$$
Using this simplification of $\GC(\tilde{f}, \mu)$ and substituting for $\Var_{\mu}(\tilde{f}^i)$ in the application of Chebyshev's inequality above, we get
$$
\mathbb{P}\left[\|\tilde{f}\| \leq t \right] \geq 1 - \dfrac{\rho \GC(\tilde{f}, \mu)}{t^2} \geq 1 - \dfrac{a_1\rho}{t^2}.
$$
As before, taking $\delta = a_1\rho/t^2$ and solving for $t$, we get $t = \sqrt{a_1\rho/\delta}$; thus, for any $\delta \in (0,1)$, it follows that 
$$
\mathbb{P} \left[\|\tilde{f}\| \leq \sqrt{a_1\rho/\delta}\right] \geq 1 - \delta.
$$
Therefore, since $\|\mathbb{E}_{\mu}[f] \| \leq a_2$, for any $f\in\mathcal{F}$ with high probability we can bound the image of $f$ within a ball in $\mathbb{R}^k$; namely,
$$
\mathbb{P}\left[\|f\| \leq a_2 + \sqrt{a_1\rho/\delta} \right] \geq 1- \delta.
$$

The rest of the argument follows from a standard ball counting argument in $\mathbb{R}^k$.
Given $\epsilon >0$, let $r := a_2 + \sqrt{a_1\rho/\delta}$ and take a maximal set of points $p_i \in B(r)$ such that $\dist(p_i, p_j) > \epsilon$ for $i\neq j$. It follows that $B_{p_i}(\epsilon/2) \cap B_{p_j}(\epsilon/2) = \emptyset$ and $\cup_i  B_{p_i}(\epsilon/2) \subset B(r(1 + \epsilon/2))$. Thus, by construction and taking volumes on both sides, $\sum_i |B_{p_i}(\epsilon/2)| \leq |B(r(1 + \epsilon/2))|$. Let $N$ denote the number of points $p_i$ and since $|B_{p_i}(\epsilon/2)| = |B(\epsilon/2)|$ for all $i \in [N]$, we get 
$$
N \leq \dfrac{|B(r(1 + \epsilon/2))|}{|B(\epsilon/2)|} = r^k (1 + 2/\epsilon)^k.
$$
Therefore, for small $\epsilon < 1$, $N \leq r^k(3/\epsilon)^k$ and thus
$$
\mathcal{N}(\{f(X) ~|~ f \in \mathcal{F}\}, \epsilon, \|\cdot\|_2) \leq \dfrac{3^k}{\epsilon^k}\left(a_2 + \sqrt{\dfrac{a_1 \rho}{\delta}}\right)^k.
$$
This completes the proof.
\end{proof}

Using this covering lemma we can now prove our main Theorem \ref{theorem:multiclass_generalization_bound}. In fact, the argument follows the same lines as the case $k=1$ with only slight modification to account for margin operator in the multi-class setting and the application of the Dudley entropy formula when bounding the empiricial Rademacher complexity. 

Let us restate the theorem and provide the proof:

\begin{theorem}
Given $a_1, a_2$ be positive reals. Let $S = \{(x_1,  y_1), \dots, (x_m, y_m)\}$ be i.i.d.~input-output pairs in $\mathbb{R}^d \times \{1,\cdots, k\}$ and suppose the distribution $\mu$ of the $x_i$ satisfies the Poincar\'e inequality with constant $\rho >0$.  Then, for any $\delta > 0$, with probability at least $1 - \delta$, every margin $\gamma > 0$ and network $f: \mathbb{R}^d \to \mathbb{R}^k$ which satisfies $\GC(f, \mu) \leq a_1$ and $\|\mathbb{E}_{\mu}(f)\| \leq a_2$ satisfy
$$
\mathbb{P}\left[\argmin_jf(x)_j \neq y \right] \leq \widehat{\mathcal{R}}_{S, \gamma}(f) + 
\dfrac{36\tilde{C}\sqrt{k\pi}}{\gamma m} + 
 3\sqrt{\dfrac{\log \frac{2}{\delta}}{2m}}
$$
where $\widehat{\mathcal{R}}_{S, \gamma}(f) =m^{-1} \sum_i \mathbbm{1}_{y_if(x_i) \leq \gamma}$ and $\tilde{C} = a_2 + \sqrt{a_1\rho/\delta}$.
\end{theorem}

\begin{proof}[Proof of Theorem \ref{theorem:multiclass_generalization_bound}]
Let $\mathcal{F}$ denote the class of differentiable maps
$$
\mathcal{F} := \{f: \mathbb{R}^d \to \mathbb{R}^k ~|~ \GC(f, \mu) \leq a_1, \|\mathbb{E}_{\mu}[f]\| \leq a_2\}
$$
and for any $\gamma >0$ define 
$$
\widetilde{\mathcal{F}}_{\gamma} := \{(x,y) \mapsto \ell_{\gamma}(-\mathcal{M}(f(x), y) ~|~ f \in \mathcal{F}\}
$$
where $\mathcal{M}(\cdot, \cdot)$ denotes the margin operator  $\mathcal{M}: \mathbb{R}^k \times \{1, \dots, k\} \to \mathbb{R}$  defined by $\mathcal{M}(v, y) = v_y - \max_{i \neq y} v_i$ and $\ell_{\gamma}:\mathbb{R} \to \mathbb{R}^+$ denotes the usual ramp loss.

Similar to the proof of Theorem \ref{theorem:binary_generalization_bound},  since $\ell_{\gamma}$ has range $[0,1]$ and it follows from classic generalization bounds based on the Rademacher complexity (e.g., Theorem 3.3 in \cite{mohri2018foundations}) that, for any $\delta > 0$, with probability at least $1 - \delta$ over the draw of an i.i.d.~sample $S$ of size $m$, we have for all $f \in \widetilde{\mathcal{F}}_{\gamma}$:
\begin{equation}\label{equation:Rademacher_generalization_multiclass}
\mathbb{E}[\ell_{\gamma}(-\mathcal{M}(f(x), y))] 
\leq 
\widehat{\mathcal{R}}_{S, \gamma}(f) + 
  2 \widehat{\mathfrak{R}}_S(\widetilde{\mathcal{F}}_{\gamma}) +   3\sqrt{\dfrac{\log \frac{2}{\delta}}{2m}}
\end{equation}
where now $\widehat{\mathcal{R}}_{S, \gamma}(f)  = m^{-1}\sum_i\ell_{\gamma}(-\mathcal{M}(f(x_i), y_i))$.

We can lower bound the left hand side of \eqref{equation:Rademacher_generalization_multiclass} (see Lemma A.4 of \cite{bartlett2017spectrally}) so that $\mathbb{P}\left[\argmax_i f(x)_i \neq y \right] \leq \mathbb{E}[\ell_{\gamma}(-\mathcal{M}(f(x), y))] $ and, via Talagrand's lemma, we can also upper bound $\widehat{\mathfrak{R}}_S(\widetilde{\mathcal{F}}_{\gamma})$ on the right hand side to get
$$
\mathbb{P}\left[\argmax_i f(x)_i \neq y \right] \leq \widehat{\mathcal{R}}_{S, \gamma}(f) + 
  \dfrac{2}{\gamma}\widehat{\mathfrak{R}}_S(\mathcal{F}) +   3\sqrt{\dfrac{\log \frac{2}{\delta}}{2m}}
$$

It remains to bound $\widehat{\mathfrak{R}}_S(\mathcal{F})$ which we can again accomplish through the Dudley entropy integral, as in the proof of Theorem \ref{theorem:binary_generalization_bound}, with only a very slight modification when using the covering number bound afforded by Lemma \ref{lemma:PI_covering_number_multiclass}. Namely, 
taking as before $\tilde{C} = \tilde{C}(a_1, a_2, \rho, \delta) := a_2 + \sqrt{a_1 \rho/\delta}$, then $\mathcal{N}\left(\left.\mathcal{F}\right|_S, \epsilon, \|\cdot\|_2 \right) \leq  (3\tilde{C}/\epsilon)^k$. Following the same argument to evaluate the integral we can obtain a comparable bound on the empirical Rademacher complexity of $\mathcal{F}$ over $S$, but now paying a cost of $\sqrt{k}$; i.e.,
$$
\widehat{\mathfrak{R}}_S(\mathcal{F})
\leq
\inf_{\alpha > 0} \left\{ \dfrac{4\alpha}{\sqrt{m}} + \dfrac{12\sqrt{k}}{m}\left(  \dfrac{3\tilde{C}\sqrt{\pi}}{2} -\alpha \sqrt{\log(3\tilde{C}/\alpha)} \right)\right\}
\leq \dfrac{18\tilde{C}\sqrt{k}\sqrt{\pi}}{m}.
$$
Thus, collecting terms we get
$$
\mathbb{P}\left[\argmin_jf(x)_j \neq y \right] \leq \widehat{\mathcal{R}}_{S, \gamma}(f) + 
\dfrac{36\tilde{C}\sqrt{k\pi}}{\gamma m} +
3\sqrt{\dfrac{\log \frac{2}{\delta}}{2m}}.
$$

\end{proof}

\section{Experiment details}\label{section:appendix_cifar_experiments}
We trained a ResNet-18 \cite{he2016deep} with SGD on CIFAR-10 and CIFAR-100 with both original and random labels. During training we trained with batch size 256 for 100,000 steps with learning rate 0.05. Here we plot the curves
for the excess risk (test accuracy - train accuracy) and compare with the geometric complexity during training. 

\begin{figure}[h]
\includegraphics[width=8.4cm]{./images/cifar10-resnet18-gc-generalization.png}
\includegraphics[width=8.4cm]{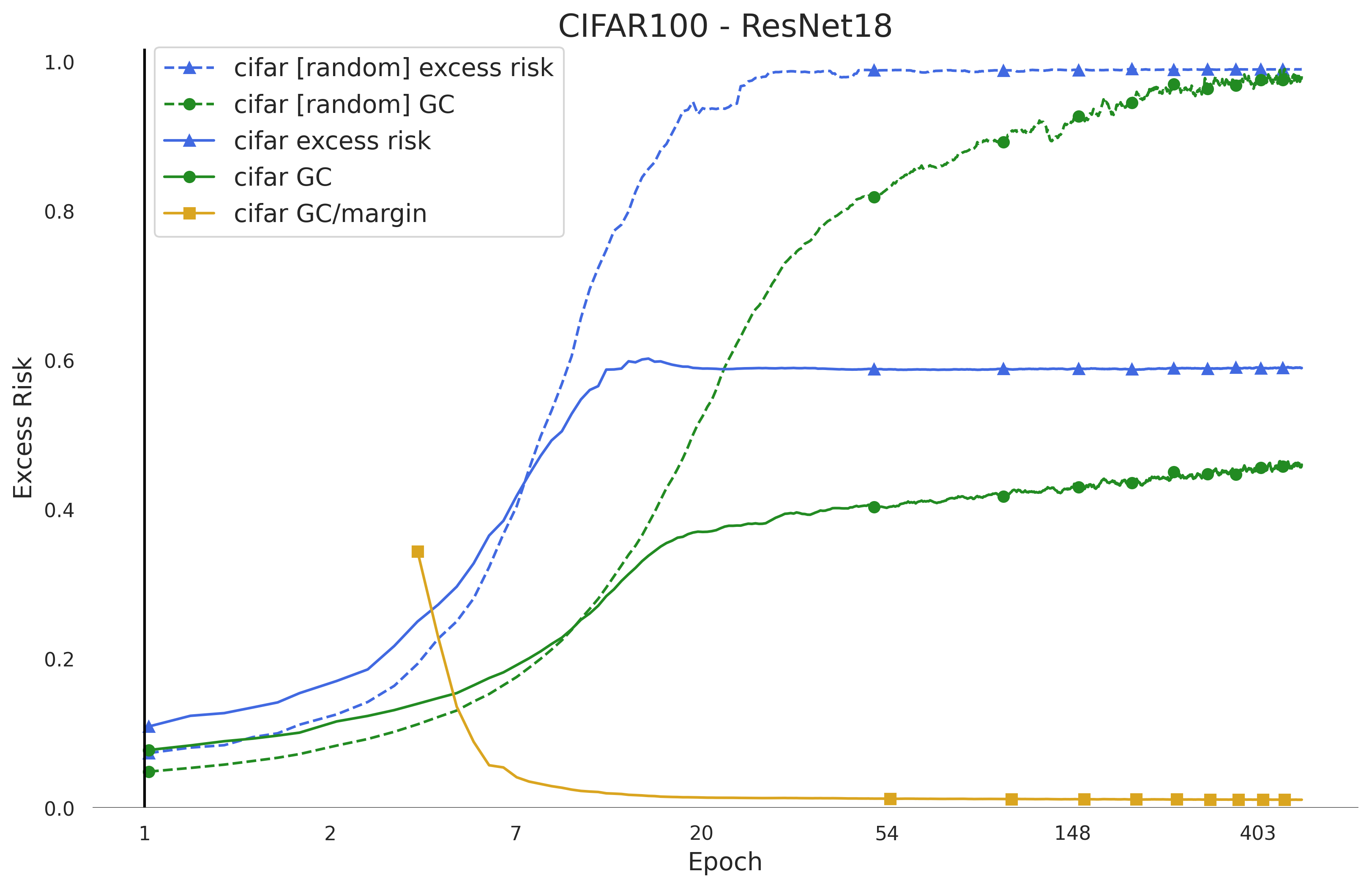}
\caption{Analysis of ResNet-18 \cite{he2016deep} trained with SGD on CIFAR-10 (left) and CIFAR-100 (right) with both original and with random labels. The triangle-marked curves plot the excess risk across training epochs (on
a log scale). Circle-marked curves track the geometric complexity ($\GC$). Note that the $\GC$ is tightly correlated with the excess risk in both settings. Normalizing the $\GC$ by the margin neutralizes growth across epochs.}
\label{figure:cifar10_appendix}
\end{figure}

\end{document}